\pdfoutput=1

\documentclass[11pt]{article}
\usepackage{naacl2021}
\usepackage{times}
\usepackage{latexsym}

\usepackage{comment}

\usepackage{amsmath}
\usepackage{amsfonts}
\usepackage{amsthm}
\usepackage{amssymb}
\usepackage{bbm}
\usepackage{booktabs}
\usepackage{tikz}

\usetikzlibrary{tikzmark}
\usetikzlibrary{bayesnet}

\usepackage{cancel}

\usepackage{algorithm}
\usepackage{algpseudocode}
\algnewcommand{\parState}[1]{\State%
  \parbox[t]{\dimexpr\linewidth-\algmargin}{\strut #1\strut}}
\algrenewcommand{\algorithmiccomment}[1]{\hskip1em$\rightarrow$ \footnotesize#1 \normalsize}

\usepackage{float}
\usepackage{multirow}

\newcommand\citepossessive[1]{\citeauthor{#1}'s\ (\citeyear{#1})}

\usepackage{microtype}

\newtheorem{prop}{Proposition}

\usepackage{cleveref}
\crefname{section}{\S}{\S\S}
\Crefname{section}{\S}{\S\S}
\crefname{table}{Tab.}{Tabs.}
\crefname{figure}{Fig.}{}
\crefname{algorithm}{Algorithm}{}
\crefname{algorithm}{Algorithm}{}
\crefname{line}{Line}{}
\crefname{appendix}{App.}{}
\crefformat{section}{\S#2#1#3}
\crefname{thm}{Theorem}{}
\crefname{def}{Definition}{}
\crefname{prop}{Proposition}{Propositions}

\newcommand{\bz}{\mathbf{z}}
\newcommand{\bw}{{\boldsymbol w}}

\newcommand{\bh}{\mathbf{h}}
\newcommand{\bb}{\mathbf{b}}
\newcommand{\bl}{{\boldsymbol \ell}}

\newcommand{\vw}{\boldsymbol{w}}
\newcommand{\vu}{\boldsymbol{u}}
\newcommand{\pgen}{p_{\mathrm{\phi}}}
\newcommand{\defn}[1]{\textbf{#1}}

\newcommand{\ent}{\mathrm{H}}

\newcommand{\calL}{\mathcal{L}}
\newcommand{\calM}{\mathcal{M}}
\newcommand{\word}[1]{\textit{#1}}
\newcommand{\concept}[1]{\texttt{#1}}
\newcommand{\cost}{\mathrm{cost}}

\newcommand{\softmax}{\mathrm{softmax}}

\newcommand{\ncode}{\mathbb{C}_{\textit{nat}}}
\newcommand{\optcode}{\mathbb{C}_\star}
\newcommand{\morphcode}{\mathbb{C}_{\textit{morph}}}
\newcommand{\graphcode}{\mathbb{C}_{\textit{graph}}}
\newcommand{\morphgraphcode}{\mathbb{C}_{\textit{morph+graph}}}

\newcommand{\cluster}{z_n}
\newcommand{\clusterrv}{Z_n}
\newcommand{\otherclusters}{\mathbf{z}_{<n}}
\newcommand{\ntokens}{\mathrm{c}_{<n}^{(z_n)}}
\newcommand{\ntables}{K_{<n}}

\newcommand{\vh}{{\boldsymbol h}}

\newcommand{\entropy}{\mathrm{H}}
\newcommand{\STOP}{\texttt{EoW}}
\newtheorem{mydef}{Definition}

\definecolor{modelpink}{HTML}{b5869D}
\definecolor{modelgreen}{HTML}{86a778}
\definecolor{modelorange}{HTML}{c99b6c}

\definecolor{codec0}{HTML}{d73570}
\definecolor{codec1}{HTML}{340371}
\definecolor{codec2}{HTML}{dbb85f}
\definecolor{codec3}{HTML}{888EAD}
\definecolor{codec4}{HTML}{00bf8b}
\definecolor{codec5}{HTML}{427657}
\definecolor{codec6}{HTML}{134679}

\everypar{\looseness=-1}

\title{How (Non-)Optimal is the Lexicon?}

\usepackage{tipa}
\newcommand{\ucambridge}{\normalfont \text{\textipa{D}}}
\newcommand{\irenes}{\normalfont \text{\textipa{7}}}
\newcommand{\ethz}{\text{\normalfont \textipa{Q}}}
\newcommand{\harvard}{\normalfont \text{\textipa{@}}}
\newcommand{\ucsb}{\normalfont \text{\textipa{N}}}
\newcommand{\mpi}{\normalfont \text{\textipa{9}}}
\newcommand{\hse}{\normalfont \text{\textipa{6}}}

\author{%
Tiago Pimentel$\thanks{~~Equal contribution}$ $^{\,,\ucambridge}$%
~\;~\;~ Irene Nikkarinen$^{*,\ucambridge,\irenes}$%
~\;~\;~ Kyle Mahowald$^{\ucsb}$ \\
\textbf{Ryan Cotterell}$^{\ucambridge,\ethz}$
~\;~\;~ \textbf{Dami\'{a}n E. Blasi}$^{\harvard,\mpi,\hse}$\\
  $^{\ucambridge}$University of Cambridge%
  ~\;~\;~\;~$^{\irenes}$Yle%
  ~\;~\;~\;~$^{\ucsb}$University of California, Santa Barbara%
  ~\;~\;~\;~$^{\ethz}$ETH Z\"{u}rich \\
  $^{\harvard}$Harvard University%
  ~\;~\;~\;~$^{\mpi}$MPI for Evolutionary Anthropology%
  ~\;~\;~\;~$^{\hse}$HSE University
   \\
  \texttt{tp472@cam.ac.uk}%
  ,~\;~ \texttt{irene.nikkarinen@gmail.com}%
  ,~\;~ \texttt{mahowald@ucsb.edu} \\%
  \texttt{ryan.cotterell@inf.ethz.ch}%
  ,~\;~ \texttt{dblasi@fas.harvard.edu}
 }

\date{}

\begin{document}
\maketitle

\begin{abstract}
The mapping of lexical meanings to wordforms is a major feature of natural languages. 
While usage pressures might assign short words to frequent meanings (Zipf's law of abbreviation), the need for a productive and open-ended vocabulary, local constraints on sequences of symbols, and various other factors all shape the lexicons of the world's languages.
Despite their importance in shaping lexical structure, the relative contributions of these factors have not been fully quantified.
Taking a coding-theoretic view of the lexicon and making use of a novel generative statistical model, we define upper bounds for the compressibility of the lexicon under various constraints. Examining corpora from 7 typologically diverse languages, we use those upper bounds to quantify the lexicon's optimality and to explore the relative costs of major constraints on natural codes. 
We find that (compositional) morphology and graphotactics can sufficiently account for most of the complexity of natural codes---as measured by code length.
\end{abstract}

\section{Introduction} \label{sec:introduction}

Communication through language can be modeled under Shannon's classic communication framework \citep{shannon1948mathematical}.
Under this perspective, linguistic utterances are \defn{codes}---which need to be decoded by a \defn{receiver} (listener) who is interested in the \defn{message} (meaning) they encode. Famously, \newcite{zipf1949human} posited that language users shape these codes so to accommodate the \defn{principle of least effort}. The most widely discussed and investigated empirical evidence for this feature is the so-called \defn{law of abbreviation}, an ostensive negative correlation between word frequency and word length \citep{zipf1935psycho,bentz2016zipf}. Communication effort decreases by encoding frequent messages in shorter words.
This correlation, however, is characteristically modest. There are many instances of short low-frequency words, like \word{wen} and \word{jib} in English,\footnote{These mean, respectively, a benign tumor on the skin and a triangular sail on a boat.} and long frequent words, like \word{happiness} and \word{anything}. While the lexicon might be shaped by economy of expression, it is clearly not fully optimized for it. There are multiple---possibly competing---reasons why this could be the case.

\begin{figure}
    \centering
    \includegraphics[width=\columnwidth]{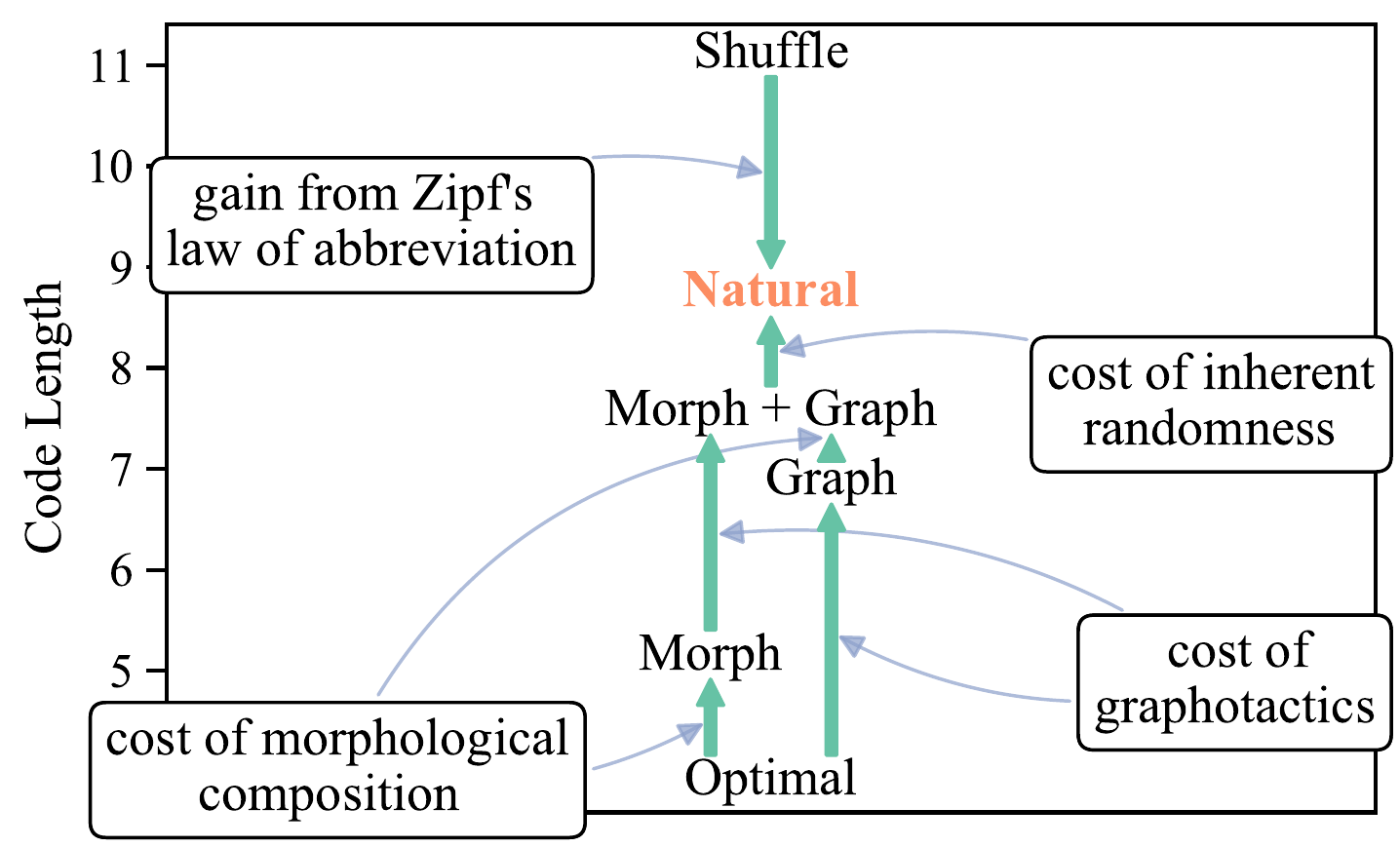}
    \caption{The average code length---under our coding schemes---on a representative language (Finnish). The distance between the baselines can be thought of as the  cost of each constraint added to the system.}
    \label{fig:costs}
    \vspace{-6pt}
\end{figure}
First of all, the sequence of speech sounds, signs, or orthographic characters that serve as building blocks in a language comply with specific rules. These are referred to as \defn{phonotactics} (in the case of speech sounds) and \defn{graphotactics} (in written language).\footnote{All languages impose these constraints on their wordforms, which might be leveraged for production and learnability \citep{vitevitch1999probabilistic,boersma1998functional}.} 
On top of these constraints, the lexicons of many languages of the world re-use sub-parts of words; these sub-parts can be productively composed to produce new meanings---which is referred to as \defn{morphological composition}. 
This largely determines the family of wordforms associated with a given basic meaning---for instance, given the wordform \emph{health} and its meaning, the nominal
morphology of English readily provides the forms for many of its derived meanings, including \emph{healthy}, \emph{unhealthy}, \emph{healthier}, etc.

Beyond these well-attested constraints, it might be argued that the negative correlation between the length of a word and its frequency is not the locus of optimization given the economy of expression pressure. Instead, wordforms might be efficiently encoding meanings based on their contextual surprise rather than frequency \citep{piantadosi2011word}.
Finally, there is no reason to expect lexicons to be \emph{fully optimized} for the economy of expression---this factor might steer languages in a given direction, but there is certainly room for non-compliance. Languages are, after all, not engineered systems but cultural artifacts.
In this paper, we examine how \emph{marginally non-optimal} the lexicon is by taking the vantage point of the law of abbreviation. 
We develop a method to quantify the role of different linguistic constraints on determining wordforms, and we produce estimates on how compressible the lexicon could be in their absence (including morphology and phonotactics/graphotactics). 
We thus define an upper bound for the compressibility of a lexicon optimized purely for word length efficiency.
\section{(Non-)Optimality in the Lexicon} \label{sec:optimality}

As stated above, our notion of optimality is derived from Zipf's principle of least effort in the form of the law of abbreviation \citep{zipf1949human,mandelbrot1953informational,ferrer2020optimal}. However, this is by no means the only theory under which wordforms are optimized for encoding their messages.

One influential hypothesis is that languages optimize for uniform information density \citep{fenk1980konstanz,aylett2004smooth,levy2007speakers}---roughly keeping the amount of information conveyed in a unit of time constant. In an information-theoretic setting, this would be equivalent to maximizing the use of a noisy channel between the speaker and an audience---keeping the transmission rate close to the channel capacity.\looseness=-1

Under this view, it is not necessarily the case that words should be as short as possible. Rather, words that are infrequent or typically less predictable in context should be longer and take more time to produce---perhaps because the increased duration makes them more robust to noise.
Consistent with this perspective, it has been shown that, in production, words with higher information content take longer to pronounce \cite{bell2003effects,jurafsky_probabilistic_2001,gahl2008time}. 
Additionally, words which are typically predictable in context are shorter than words which are less predictable in context \cite{piantadosi2011word}. 

On another note, a purely coding-theoretically efficient language could make the lexical codes context dependent \citep{piantadosi2012communicative}, since context often disambiguates words \citep{dautriche2018learning,pimentel-etal-2020-speakers}.
Additionally, the meaning or message being conveyed by a given word might bias its form. Within languages, there seems to be a pressure for more semantically similar words to also be more phonologically similar \citep{monaghan_how_2014,dingemanse2015arbitrariness,dautriche2017wordform,pimentel2019meaning}. 
Across languages, words for the same referents exhibit detectable patterns in term of their phonological makeup \cite{blasi2016sound}, phonotactics \cite{pimentel-etal-2021-finding}, as well as word length \cite{lewis2016length}---this is driven by semantic features such as size, quality or complexity. 
Finally, there is a cross-linguistic tendency for lexicons to place higher surprisal in word-initial segments \citep{son2003efficient,son2003information,king2020greater,pimentel-etal-2021-disambiguatory} making words more constrained in their choice of final segments.
These aspects of language might also collide with a purely Zipfian conception of lexicon optimality.

In this work, however, we consider optimality exclusively in the Zipfian sense of compressibility, and we ask how far natural language lexicons are from accommodating to this paradigm. We build a number of models that differ in relation to whether they accommodate to the law of abbreviation, to compositional morphology and to graphotactics. The comparison among these systems allows us to explore the extent to which each part of the linguistic system contributes to the overall cost of the linguistic code. It should be noted, though, that the consequences of unmodeled sources of structure in the lexicon (such as persistent sound-meaning associations or the adaptation of the code to surprisal effects) will forcibly be confounded with the overall lack of fit between our models and the data.

The \textbf{morphological cost}
---i.e. the cost of morphology to a code's length---is associated with the fact that, across many languages, words are often constructed of meaningful sub-parts that are productively reused across the lexicon. 
Practically, this means that the wordforms of different meanings might not be independent if they overlap in a particular dimension that is captured by the morphology of the language. For instance, most wordforms that express two or more referents of a kind share a word-final suffix \word{-s} in English (\word{tower\underline{s}}, \word{cat\underline{s}}, \word{idea\underline{s}}, etc). We treat this cost by considering optimal codes where the basic unit in the lexicon is not the word but sub-pieces, as determined by the unsupervised morphological parser Morfessor
\citep{creutz2007unsupervised,smit-etal-2014-morfessor}.\footnote{We also present results using the additional sub-word tokenizers: byte pair encoding \citep{gage1994new,sennrich-etal-2016-neural} and word piece \citep{schuster2012japanese}. See \citet{bostrom-durrett-2020-byte} for a discussion of the tradeoffs of these schemes, in terms of performance and compressibility.} Under this regime, a word like \word{unmatched} is parsed into the tokens \word{un}, \word{match}, and \word{ed}.

The \defn{graphotactic cost} concurrently imposes a set of additional constraints, determining which sequences of grapheme segments can constitute a valid wordform in a given language. 
While the main driver of these lexical constraints is actually phonotactics---which imposes rules dictating the possible \emph{phoneme} sequences---we focus on graphotactics because our object of study is written language corpora. 
The degree to which phonotactics and graphotactics mirror each other vary substantially across languages; thus, in this work (which uses corpora from Wikipedia) we make our claims about language in the written modality and leave it to future work to generalize this work to the phonological domain. 
This could be done by applying the same method to phonemic representations of words.\looseness=-1

\section{A Coding-theoretic View of the Lexicon}

This paper treats the \defn{lexicon}, which we define as a set of pairs: $\calL = \{(m_n, \bw_n)\}_{n=1}^N$.
In general, this set will be infinite; $m_n$ refers to a \defn{lexical meaning}, taken from an abstract set $\calM$, and $\bw_n$ refers to a \defn{wordform},
taken from $\Sigma^*$, the Kleene closure of a grapheme alphabet $\Sigma$.\footnote{This alphabet is augmented with an end-of-word symbol.}
When the exact index is unnecessary in context, we will drop the subscripted $n$; and we make use of uppercase letters to refer to random variables (e.g. $M$ or $W$) where necessary.
We will write meanings in typewriter font, e.g. \concept{cat}, and
wordforms in italics: \word{cat} (English), 
\word{kissa} (Finnish).
\looseness=-1

Viewing the lexicon from a coding-theoretic perspective, we consider the mapping from meaning to form as a \defn{code}: $\mathbb{C} : \calM \rightarrow \Sigma^*$.
Every language comes endowed with a \defn{natural code} $\ncode$, which is the observed mapping from lexical meanings to forms. 
As an example, consider the meaning \concept{cat} and its Finnish form: 
we have $\ncode(\concept{cat}) = \word{kissa}$.
The topic of interest in this paper is the efficiency of language's natural codes. 

\paragraph{The space of meanings and lexical ambiguity.}
The space of meanings $\calM$ is non-trivial to define, but could be operationalized as $\mathbb{R}^d$, which is infinite, continuous and uncountable \citep{pilehvar2020embeddings}.
Meanwhile, the space of wordforms $\Sigma^*$ is also infinite,  but discrete and countable.
As such, many meanings $m_n$ must be mapped to the same form, resulting in lexical ambiguity. See \citealt{pimentel-etal-2020-speakers} for a longer discussion on these operationalizations.
In this work, though, we do not engage with such ambiguity, considering $\calM$ as an abstract set of meanings, each of which defined by a distinct wordform---i.e. the code $\ncode$ is a bijection.
A consequence of this strategy is that we take the space of meanings to be infinite, but discrete and countable; we only distinguish as many meanings as there are words, therefore, we end up with a countable number of meanings.
Additionally, by considering a distinct meaning $m_n$ for each wordform $\bw_n$ in the lexicon, we only consider codes with as much lexical ambiguity as in the original language.\footnote{Lexical ambiguity allows the mapping of multiple meanings to the same wordform and, in doing so, it enables the preferential re-use of short words \citep{piantadosi2012communicative}. Thus, the mapping of multiple meanings to the same form could be a source of efficiency in the lexicon \citep{fenk2008complexity,ferrer2018origins,casas2019polysemy,trott2020human,xu2020conceptual}. Nonetheless, we do not treat it explicitly here.}

\subsection{Words as Meanings} \label{sec:words_meaning}

The \defn{unigram distribution}  represents the frequency of each wordform in a text, i.e. the probability of a token without conditioning on context $p(W=\word{kissa})$.
In this work, though, we assume the unigram distribution is a distribution over $\calM$, e.g. $p(M=\concept{cat})$---this way we can analyze how changing the code $\mathbb{C}$ would affect its efficiency.

As stated above, though, we take $\ncode$ to be a bijection.
Such an assumption implies there is a deterministic function from wordforms to meanings in a specific lexicon $\ncode^{-1}(\bw)=m$.
Probabilistically speaking, we write
\begin{align}
    p(M = m \mid W=\bw) = \mathbbm{1}\Big\{m=\ncode^{-1}(\bw)\Big\} \\
    p(W = \bw \mid M=m\} = \mathbbm{1}\Big\{\bw=\ncode(m)\Big\}
\end{align}
This mapping implies
\begin{align} \label{eq:distribution_equiv}
    p(M=m_n) &= \sum_{\bw \in \Sigma^*} p(M=m_n, W=\bw)  \\
    &= p(W=\bw_n) \nonumber
\end{align}
Given this equality, we can reduce the problem of estimating the unigram distribution over meanings $p(m)$ to the one over wordforms $p(\bw)$.

\subsection{Code-length and optimality}

As stated above, we assume the unigram distribution to be a distribution over $\calM$. 
We now define the \defn{cost} of a code as its expected length:
\begin{equation}\label{eq:cost_code}
  \cost(\mathbb{C}) = \sum_{m \in \calM} p(m)\,|
 \mathbb{C}(m)|
\end{equation}
A smaller cost, then, implies a more \defn{efficient} code.
The famous \defn{source-coding theorem} of \newcite{shannon1948mathematical} gives us a theoretical limit on coding cost:
\begin{equation}\label{eq:shannon}
    \ent(M) \leq \cost(\optcode) < \ent(M) + 1
\end{equation}
where we define $\optcode$ to be the most efficient code, and where $\ent(M)$ is the entropy of distribution $p$:%
\begin{equation}
    \ent(M) = \sum_{m \in \calM} p(m)\, \log_{|\Sigma|} \frac{1}{p(m)}
\end{equation}
According to the source-coding theorem, if we know the true distribution $p$ over lexical meanings, then we know how to optimally code them. 
This turns the problem of estimating the efficiency of the lexicon into the one of estimating the entropy of an unknown discrete distribution $p$, a well-defined task with a pool of previous work \citep{miller1955note,antos2001convergence,paninski2003estimation,archer2014bayesian}.
Because the distributions over wordforms and meanings are equivalent, we estimate the entropy $\ent(M)$ using wordforms:%
\begin{align} \label{eq:entropy_equiv}
\ent(M) &= \ent(W) = \sum_{\bw \in \Sigma^*} p(\bw)\, \log_{|\Sigma|} \frac{1}{p(\bw)} 
\end{align}

\subsection{Finite and Infinite Support}

This section reviews a few technical results as regards the construction of codes from a probability distribution. 
If $p$ had finite support---i.e. there were a finite set of possible meanings or wordforms---a simple  Huffman encoding \cite{huffman1952method} would give us an optimal code for our lexicon.
However, this is not the case---$p(\bw)$ has support on all of $\Sigma^*$---so we might need a more complex strategy to get such a code.
\citet{linder1997existence} proved the existence of an optimal encoding for a distribution with infinite support, given that it has finite entropy.
\begin{prop} \label{prop:bound_cost}
If distribution $p(\bw)$ has finite entropy, i.e.
$\ent(W) < \infty$, then there exists an optimal encoding for it such that: $\cost(\optcode) < \ent(M) + 1$.
\end{prop}
\begin{proof}
See \citet{linder1997existence}.
\end{proof}
\noindent Luckily, under a weak assumption, this is the case for a well-trained language model.

\begin{mydef}
Language model $p(\bw)$ is $\varepsilon$-smooth if for all histories $\vh \in \Sigma^*$ we have $p(\STOP \mid \vh) \geq \varepsilon$.\footnote{Under this assumption our language model is also consistent, as defined by \citet{welleck-etal-2020-consistency}---sequences with infinite length have asymptotically zero probability mass.}
\end{mydef}
\noindent This fairly weak assumption states that partial wordforms have a lowerbound on their probability of ending. 
As such, there is an upperbound on the probability of a wordform which decreases exponentially with its length.
Armed with this assumption, we can now show that any $\varepsilon$-smooth language model has a finite entropy. 
\begin{prop} \label{prop:bound_entropy}
If a language model $p(\bw)$ is $\varepsilon$-smooth, then its entropy is finite, i.e.
$\ent(W) < \infty$.
\end{prop}
\begin{proof}
See \cref{app:proof_entropy}.
\end{proof}

\noindent Safe-guarded by \cref{prop:bound_cost,prop:bound_entropy}, we now train a model to capture the unigram distribution. 
We will then use this model to estimate the code-length of an optimal lexicon.

\section{Modeling the Unigram Distribution and its Challenges}

\citepossessive{zipf1935psycho} %
law states that the frequency of a word in a corpus is inversely proportional to its rank, resulting in a power-law distribution where a small subset of the words dominate the corpus.
As such, na{\"i}vely training a character-level model on a language's tokens (i.e. predicting non-contextual wordforms with their natural corpus frequencies) 
would be unlikely to capture morphological regularities \citep{Goldwater2011}. 
Furthermore, it would burden the model to learn
a mostly arbitrary assignment between form and frequency. 
As an example, the English verb \word{make} is much more common than the nouns \word{cake} and \word{lake}, even if graphotactically they may be equally probable.

A closer inspection of English shows that most frequent words tend to come from closed lexical classes including articles, pronouns, prepositions, and auxiliaries, such as \textit{the}, \textit{of}, \textit{it} and \textit{be} \cite{sinclair1999way}. These words tend to be short and manifest fossilized graphotactics (and phonotactics) as well as a more abundant prevalence of otherwise rare segments, such as the voiced and voiceless dental fricatives (orthographically expressed with \word{th}).
These rare segments would be overrepresented in such a na\"ive training regime, making it hard for the character-level model to correctly represent the language's graphotactics.

In order to address the problem of skewed frequencies, we use a novel neuralization of \citepossessive{Goldwater2011} two-stage model to capture the unigram distribution.
This model consists of two components: a wordform \defn{generator} and a token frequency \defn{adaptor}.
The generator is a character-level model which produces wordforms, for which we use an LSTM;\footnote{LSTMs have been shown to be able to model phonotactics well by \newcite{pimentel-etal-2020-phonotactics}, and so we expect them to also work well with graphotactics.} this model should place similar probability mass on graphotactically ``good'' wordforms, such as \word{make}, \word{cake}, and \word{lake}.
Meanwhile, the adaptor sets the frequency with which these wordforms will appear as tokens.
Following \citeauthor{Goldwater2011}, we base our adaptor on the Pitman--Yor Chinese restaurant process \cite[PYCRP;][]{Pitman1997}, which allows the adaptor to model a power-law distribution; this model is then responsible for capturing the fact that \word{make} is a more frequent token than \word{cake}, and \word{lake}.
\subsection{A Two-stage Model}

The generative process of our two-stage model is presented graphically in \cref{fig:two_stage_model}.
Our generator is a character-level LSTM language model, which generates a 
potentially infinite number of i.i.d. wordforms $\{\bl_k\}_{k=1}^K$.
Independently, the PYCRP adaptor assigns each observed token in a dataset to a cluster $\{z_n\}_{n=1}^N$.
In the literature, the value of $z_n$ is
the ``table assignmment'' of the $n^{\text{th}}$ token.
These clusters are then used as lookup indices to the wordforms, producing the observed word tokens $\{\bw_n\}_{n=1}^N$ where $\bw_n=\bl_{z_n}$.
In general $N \gg K$, so tokens with the same wordform are grouped in few clusters. 
In this way, the adaptor sets the frequency with which wordforms appear as tokens in a corpus by defining each cluster's probability.\looseness=-1

\paragraph{Generating Wordforms.}
As mentioned above, wordforms are sampled i.i.d. from a distribution $\pgen$ over strings defined by the generator.
Specifically, this distribution over forms is defined as follows:
\begin{equation} \label{eq:wordform_prob}
    \pgen (\bl) = \prod_{t=1}^{|\bl|} \pgen(\ell_t \mid \bl_{<t}) 
\end{equation}
where $\bl$ is a vector of characters forming a word and $\ell_t$ is its $t^{\text{th}}$ character.\footnote{We note two subscripts are used here: $k$ refers to the $k^{\text{th}}$ wordform, while $t$ indexes the $t^{\text{th}}$ character in the wordform.}
Each of these characters is encoded with a lookup vector, producing representations $\mathbf{e}_t \in \mathbb{R}^{d_1}$ where $d_1$ is the embedding size.
These embeddings are then used as input to an LSTM \cite{hochreiter1997long},
producing the representations $\bh_t \in \mathbb{R}^{d_2}$, where $d_2$ is the size of the LSTM's hidden layer.
The LSTM output is further used to obtain the distribution over potential characters:
\begin{align} \label{eq:lstm_prob}
    \pgen&(\ell_t \mid \bl_{<t})  = \softmax(\mathbf{W}\,\bh_t + \bb)
\end{align}
In this equation, both $\mathbf{W} \in \mathbb{R}^{|\Sigma| \times d_2}$ and $\bb \in \mathbb{R}^{|\Sigma|}$ are learnable parameters and the zero vector is used as the initial hidden state $\bh_0$.
The distribution $\pgen$, representing the generator, is then used to generate the 
set of wordforms $\{\bl_k\}_{k=1}^K$, which is expected to represent the graphotactics and morphology of the language. Notedly, these wordforms do not explicitly capture any notion of token frequency.\footnote{This generative process allows the same wordform to be sampled multiple times, as they are generated 
i.i.d.
}

\begin{figure}
    \centering
    \includegraphics[width=\columnwidth]{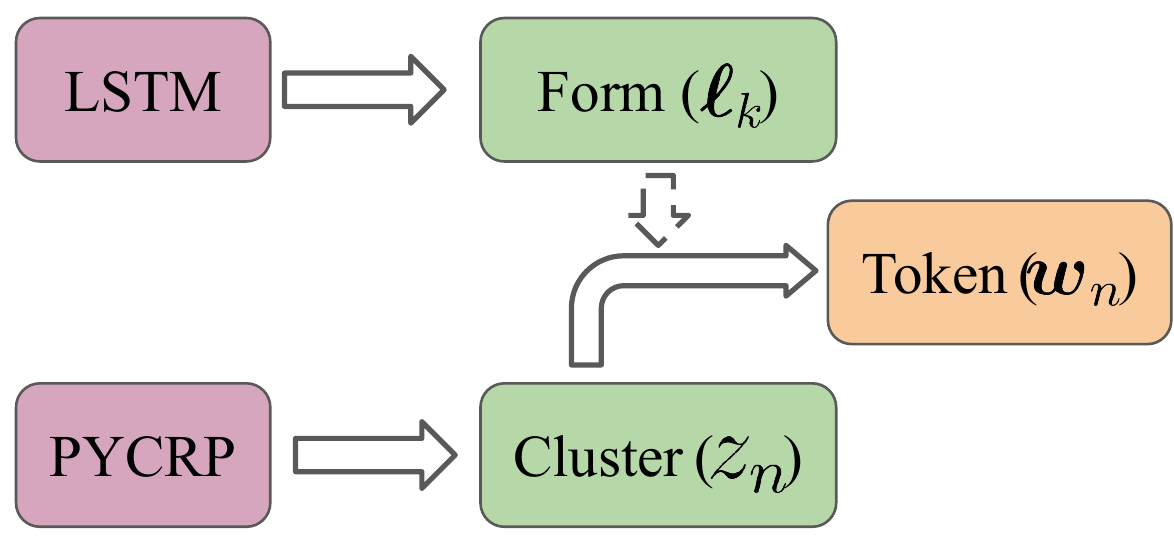}
    \caption{A diagram of the two-stage model. The LSTM generates wordforms ($\bl_k$).
    The PYCRP samples cluster assignments ($z_n$). Cluster assignments are then used to lookup a form for each token ($\bw_n=\bl_{z_n}$).
    In this Figure, models are in {\color{modelpink} \textbf{magenta}}, latent variables in {\color{modelgreen} \textbf{green}} and observed variable in {\color{modelorange} \textbf{orange}}.
    }
    \label{fig:two_stage_model}
    \vspace{-3pt}
\end{figure}

\paragraph{Adapting Word Frequencies.}
The adaptor is responsible for modeling the word frequencies, and it has no explicit notion of the wordforms themselves.
The PYCRP assigns each token $n$ to a cluster $z_n$. 
Each cluster $z_n$, in turn, has an associated wordform $\bl_{z_n}$, sampled from the generator. Consequently, all instances in a cluster share the same wordform. 
The probability of an instance $n$ being assigned to cluster $z_n$ is defined as follows:
\begin{align}
    p&(\clusterrv = \cluster \mid \,\otherclusters) \\
    &\propto 
    \begin{cases}
        \ntokens - a \, &  1 \leq z_n \leq \ntables \,\, {\color{gray}\textit{(old cluster)}} \\
        a \cdot \ntables + b \, & z_n = \ntables + 1 \,\, {\color{gray}\textit{(new cluster)}}
    \end{cases} \nonumber
\end{align}
In this equation, $\ntables$ is the current number of populated clusters; while $\ntokens$ is the number of instances currently assigned to cluster $z_n$.
The PYCRP has two hyperparameters: $0 \leq a < 1$ and $b \geq 0$. The parameter $a$ controls the rate in which the clusters grow \cite{Teh2006}, while $b$ controls an initial preference for dispersion.
Together, these ensure the formation of a long-tail---%
concocting a power-law distribution for the cluster frequencies. 
This property allows a cluster with wordform \word{make}, for example, to have an exponentially larger frequency than its graphotactic neighbor \word{cake}.

\paragraph{Modeling Word Tokens.}
Finally, given the set of wordforms and the cluster assignments, defining the form associated with a token is deterministic. 
Since each cluster only contains instances of one wordform, the form of a token is defined looking up the label of the cluster it was assigned to $\bl_{z_n}$:
\begin{align}
    p(W_n = \bw_n \mid \cluster, \> \bl) &= \mathbbm{1}\{\bw_n = \bl_{z_n}\}
\end{align}
This way, the adaptor captures the frequency information of the words in the corpus---whereas the generator can focus on learning the language's graphotactics and morphology.

\paragraph{Model training.} 
Unfortunately, we cannot directly infer the parameters of our model with a closed form solution. We thus use a solution akin to expectation maximization \cite{Wei1990}: We freeze our LSTM generator while learning the PYCRP parameters, and \emph{vice versa}. 
The PYCRP is trained using Gibbs sampling. 
For each token, we fix all cluster assignments $\mathbf{z}_{-n}$ except for one $z_n$. 
This cluster is then re-sampled from the marginal $p(\clusterrv = \cluster \mid \mathbf{z}_{-n}, \bl, \bw_n)$, where we have access to $\bw_n$ since it is an observed variable.
During this optimization small clusters may vanish, and new clusters $z_n=K+1$ (previously with no samples) may be created. 
This procedure, thus, may also produce new sets of wordforms $\{\bl_k\}_{k=1}^{K'}$, composed of the populated clusters' labels (where $K'$ is the new number of clusters).
We assume the distribution of these wordforms---which have dampened frequencies---to be more balanced than in the original full set of word tokens. 
The LSTM is trained using stochastic gradient descent, minimizing the cross-entropy of precisely this set of cluster's wordforms. 
As such, it is expected to be a more representative model of a language's graphotactics; the irregular common words are less dominant in this training set.
We give a longer explanation of our model training procedure, together with the used hyperparameters, in \cref{sec:model_training}.
\subsection{A More Intuitive Explanation}

Despite its slightly odd formulation, the two-stage model has an intuitive interpretation. Once we have learned (and fixed) its parameters,
we obtain the marginal probability of a wordform as:
\begin{align} \label{eq:interpolation_unigram}
    &p(\bw)=   \\
    &\underbrace{\frac{ c_\bw - \overbrace{n_\bw \cdot a}^{\text{smoothing factor}}}{|\bz| + b}}_{\text{smoothed unigram frequencies}} + \underbrace{\frac{(a \cdot K + b)}{|\bz| + b}}_{\text{interpolation weight}} \cdot \underbrace{\pgen(\bw)}_{\text{LSTM}} \nonumber
\end{align}
In this equation, $c_\bw$ is the count of tokens with form $\bw$ in the training set, while $n_\bw$ is the number of distinct clusters with this same form.
The model interpolates between a smoothed unigram corpus frequency and the probability an LSTM gives the analyzed wordform. 
This interpolation enables the model to place a non-zero probability mass on all possible wordforms---thus modeling an open vocabulary and having infinite support---while also placing a large probability mass on frequent wordforms.
Furthermore, the smoothing factors per word type, together with the interpolation weight, are 
holistically learned by the PYCRP model using the training set.\footnote{Our model consistently produced lower cross-entropies (on held out tokens) to the ones of an LSTM baseline na\"ively trained on a language's tokens.}
\section{Experimental Setup}

\subsection{Evaluation}

The value in which we are interested in this work is the expected cost of a code, given in \cref{eq:cost_code}.
We can easily estimate this value for a natural code by using its sample estimate:
\begin{equation} \label{eq:code_nat}
    \cost(\ncode) \approx \frac{1}{N} \sum_{n=1}^N |\ncode(m_n)| = \frac{1}{N} \sum_{n=1}^N |\bw_n|
\end{equation}
For an optimal code, we can upperbound it using the entropy of the distribution, while the entropy itself can be upperbounded by the cross-entropy of a model on it.
We can compute this upperbound with a sample estimate of the cross-entropy:
\begin{align} \label{eq:code_opt_old}
    \cost(\optcode)
    &\le \ent(W) + 1 \le \ent_\theta(W) + 1 \\
    &\lesssim \frac{1}{N} \sum_{n=1}^N \log_{|\Sigma| } \frac{1}{p_\theta(\bw_n)} + 1 \nonumber
\end{align}
In practice, we get a tighter estimate by using the \citet{shannon1948mathematical} code's lengths directly:
\begin{align} \label{eq:code_opt}
    \cost(\optcode)
    \lesssim \frac{1}{N} \sum_{n=1}^N \left\lceil \log_{|\Sigma| } \frac{1}{p_\theta(\bw_n)} \right\rceil
\end{align}
where $\lceil \cdot \rceil$ is the ceiling operation.

\subsection{Morphological Constraints}

As mentioned in \cref{sec:optimality}, we use Morfessor \citep{smit-etal-2014-morfessor} to tokenize our corpus into morphological units. 
Morfessor is a method for finding morphological segmentations from raw text data. 
As an unsupervised model, Morfessor is inherently noisy, but we take it as a  proxy for a language's morphological segmentation. 
To compare the robustness of our results across different unsupervised segmentation algorithms, though, we also run our experiments using byte pair encoding \citep[BPE;][]{gage1994new,sennrich-etal-2016-neural} and WordPieces~\citep{schuster2012japanese}.
\looseness=-1

We train Morfessor on all pre-tokenized sentences in our language-specific Wikipedia corpus (described in \cref{sec:dataset}). With this pre-trained model in hand, we tokenize all words in our training, development and test sets. We get a set of morpheme tokens $\{\vu_{n,j}\}_{j=1}^{J_n}$ for each word $\vw_n$, where this word is split into $J_n$ morphological units.\looseness=-1

We can now get the optimal length of a morphologically constrained code. With this in mind, we first train a fresh version of our two-stage model on the full set of morphological unit tokens---i.e. $\{\vu_{n,j} \mid n \le N, j \le J_n\}$, as opposed to the set of full word tokens, $\{\vw_n\}_{n=1}^N$.
We estimate the length of this code with the following equation:
\begin{align} \label{eq:code_morph}
    \cost(\morphcode) 
    &= \sum_{m \in \calM} p(m) \left|\morphcode(m)\right| \\
    &\lesssim \frac{1}{N} \sum_{n=1}^N \sum_{j=1}^{J_n} \left\lceil \log_{|\Sigma| } \frac{1}{p_\theta(\vu_{n,j})} \right\rceil \nonumber
\end{align}
Note that this cost estimate is still the average code-length per word token, as such we take the expectation over the meanings distribution. 
Each word's code-length, though, is now defined as the sum of the length of each of its constituent morphemes.

\subsection{Graphotactic Constraints}
The second linguistic constraint we would like to impose on our codes is graphotactic well-formedness---i.e. we wish our code to be composed only by sequences of characters that comply with the regularities observed in the language, such as e.g. vowel harmony, syllable structure, or word-initial and word-final constraints. 
We use our generator LSTM for this. 
As mentioned before, this model is trained on wordforms with dampened frequencies---we thus expect it to learn a language's graphotactic patterns above a minimum quality threshold.
We use this character-level model to sample (without replacement) as many unique wordforms as there are word types in that language (see \cref{tab:dataset_types} in \cref{sec:dataset_sizes}).%
\footnote{
Unfortunately, our LSTMs use a softmax non-linearity to assign probabilities and, as such, can't produce zeros.
Furthermore, due to the compositional nature of wordform probabilities (see \cref{eq:wordform_prob}), short implausible forms may have larger probability mass than long plausible ones. To mitigate this effect, when sampling wordforms we impose a minimum threshold of $0.01$ on each transition probability $p(\ell_t \mid \bl_{<t})$.
}
We assign each of these sampled wordforms $\vw'_n$, ordered by word length, to one of the languages meanings $m_n$, inversely ordered by unigram probability, i.e. $\graphcode(m_n)=\vw'_n$---thus generating an optimally Zipfian frequency--length correlation. With these assignments, we estimate the cost of a graphotactically constrained code:
\begin{align} \label{eq:code_graph}
    \cost(\graphcode) 
    \approx \frac{1}{N} \sum_{n=1}^{N} |\vw'_n|
\end{align}
Analogously, with the generator trained on morpheme units we get an optimal code under both morphological and graphotactic constraints.
\begin{align} \label{eq:code_morph_graph}
    \cost(\morphgraphcode) \approx \frac{1}{N} \sum_{n=1}^{N} \sum_{j=1}^{J_n} |\vu'_{n,j}|
\end{align}

\begin{figure*}
    \centering
    \includegraphics[width=\textwidth]{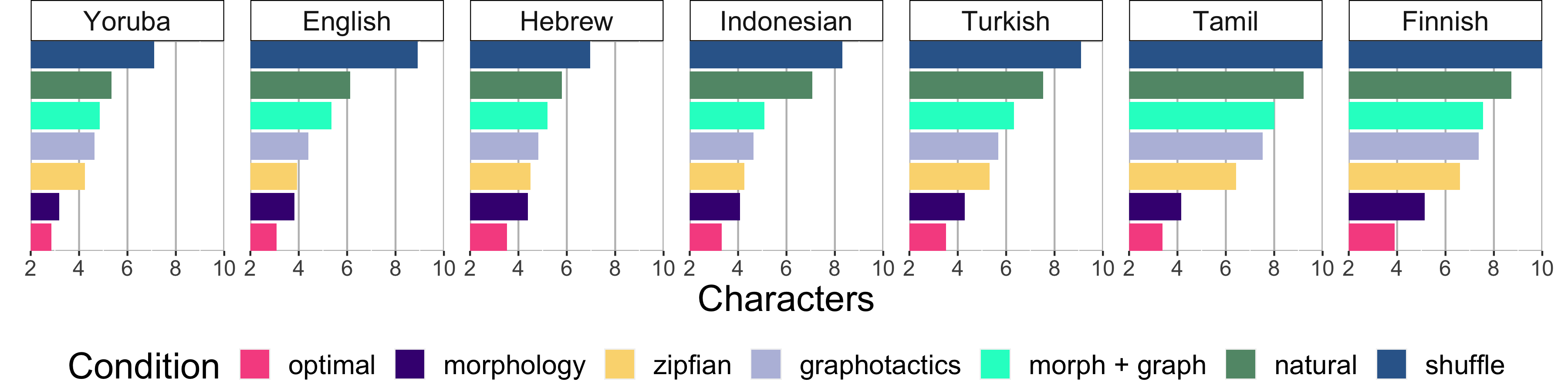}
    \vspace{-18pt}
    \caption{Bar plots of the code lengths under different constraints. In this plot, morphology is constrained through the use of Morfessor segmentation. Length in the shuffle condition for Tamil and Finnish exceed the scale (11.5 and 11 respectively).
    {\color{codec0} \textbf{Optimal}}, {\color{codec1} \textbf{Morph}}, {\color{codec2} \textbf{Zipfian}}, {\color{codec3} \textbf{Graph}}, {\color{codec4} \textbf{Morph + Graph}}, {\color{codec5} \textbf{Natural}}, {\color{codec6} \textbf{Shuffle}}.
    }
    \label{fig:bar_plot}
    \vspace{-5pt}
\end{figure*}

\subsection{Dataset} \label{sec:dataset}

We use Wikipedia data in our experiments. The data is preprocessed by first splitting it into sentences and then into tokens using SpaCy's language-specific sentencizer and tokenizer \cite{spacy2020}.
After this, all punctuation is removed and the words are lower-cased.
We subsample (without replacement) one million sentences of each language for our experiments, due to computational constraints. 
We then use an 80-10-10 split for our training, validation and test sets.

We choose typologically diverse languages for our experiments, each from a different language family: English, Finnish, Hebrew, Indonesian, Tamil, Turkish and Yoruba.\footnote{Dataset statistics are presented in \cref{sec:dataset_sizes}.} 
These languages vary in their graphotactic tendencies and morphological complexity.
In order to improve our data quality, we hand-defined an alphabet for each language and filter sentences with them, only considering sentences consisting exclusively of valid characters.\footnote{We define these sets of valid characters based on Wikipedia entries for the languages and the alphabets available in \url{https://r12a.github.io/app-charuse/}.}

\subsection{Summary}

In this paper we consider the following codes:

\vspace{-3pt}
\paragraph{\color{codec0} Optimal.} An information-theoretically optimal code under our two-stage model, estimated as defined by \cref{eq:code_opt}. This is our most compressed code and does not include either morphlogical or graphotactic contraints.

\vspace{-3pt}
\paragraph{\color{codec1} Morph.} A morphologically constrained code, as defined by \cref{eq:code_morph}. 

\vspace{-3pt}
\paragraph{\color{codec3} Graph.} A code constrained by graphotactics, as defined by \cref{eq:code_graph}. 

\vspace{-3pt}
\paragraph{\color{codec4} Morph+Graph.} A code constrained by both morphology and graphotactics; defined by \cref{eq:code_morph_graph}.

\vspace{-3pt}
\paragraph{\color{codec5} Natural.} The natural code---equivalent to the average token length and defined by \cref{eq:code_nat}. This is the code length actually observed in our corpora.

\vspace{-3pt}
\paragraph{\color{codec2} Zipfian.} A code estimated by re-pairing wordforms with meanings based on their frequencies; we then compute \cref{eq:code_nat} in this new code. This would be equivalent to the natural code length if lexicons had a perfect word length--frequency correlation (i.e., a Spearman's rank correlation of 1).

\vspace{-3pt}
\paragraph{\color{codec6} Shuffle.} A code estimated by randomly re-pairing wordforms with meanings and computing \cref{eq:code_nat} in this new code. This would be equivalent to the natural code length if Zipf's law of abbreviation did not exist, i.e. lexicons had no word length--frequency correlation.

\section{Results}

\begin{table*}[t]
    \centering
\resizebox{\textwidth}{!}{%
    \begin{tabular}{lccccccccccc}\toprule
        & &\multicolumn{3}{c}{{\color{codec1} \textbf{Morph}}} &  & \multicolumn{3}{c}{{\color{codec4} \textbf{Morph + Graph}}} \\ \cmidrule{3-5} \cmidrule{7-9}
        Language & {\color{codec0} \textbf{Optimal}} & {\color{codec1} \textbf{Morfessor}} & BPE & WordPieces & {\color{codec3} \textbf{Graph}} & {\color{codec4} \textbf{Morfessor}} & BPE & WordPieces & {\color{codec2} \textbf{Zipfian}} & {\color{codec5} \textbf{Natural}} & {\color{codec6} \textbf{Shuffle}} \\
         \midrule
English & 3.09 & 3.82 & 3.34 & 3.31 & 4.39 & 5.34 & 4.67 & 4.70 & 3.93 & 6.11 & 8.91 \\
Finnish & 3.89 & 5.13 & 4.94 & 4.95 & 7.37 & 7.55 & 7.60 & 7.65 & 6.59 & 8.72 & 10.97 \\
Hebrew & 3.52 & 4.38 & 3.98 & 3.99 & 4.82 & 5.19 & 4.95 & 4.88 & 4.50 & 5.79 & 6.97 \\
Indonesian & 3.31 & 4.08 & 3.67 & 3.66 & 4.63 & 5.08 & 5.02 & 4.95 & 4.25 & 7.06 & 8.30 \\
Tamil & 3.38 & 4.15 & 4.07 & 4.01 & 7.52 & 8.01 & 8.16 & 8.22 & 6.41 & 9.21 & 11.48 \\
Turkish & 3.52 & 4.28 & 4.12 & 4.03 & 5.67 & 6.31 & 5.98 & 5.93 & 5.31 & 7.52 & 9.09 \\
Yoruba & 2.84 & 3.18 & 3.00 & 2.97 & 4.63 & 4.85 & 4.69 & 4.61 & 4.24 & 5.34 & 7.10 \\
        \bottomrule
    \end{tabular}
}
\vspace{-5pt}
    \caption{
    The average code lengths under the different coding schemes. 
    }
    \label{tab:avg_code_lengths}
    \vspace{-8pt}
\end{table*}

The average length for each considered code is presented in \cref{fig:bar_plot} and \cref{tab:avg_code_lengths}.
As expected, we find that the average code length across natural languages is shorter than the shuffle condition and longer than the optimal condition. 
Interestingly, the codes produced by the other conditions investigated here also have the same identical order across all analyzed languages.
\looseness=-1

Adding morphological constraints on the code incurs no more than one extra character over the optimal condition---except for Finnish, for which the cost of morphology is slightly above one character.
Notably, the use of unsupervised morphological segmentation may introduce some noise into our measurements.
Consistently with our expectations, though, Yoruba (a morphologically poor language) pays the smallest cost for its morphology, while Finnish (a morphologically rich one) pays the largest.\looseness=-1

BPE and WordPiece systematically produce shorter codes than Morfessor. This is sensible, since the first two would keep most frequent wordforms intact, generating a unique code for each of them.
This would lead to codes in which the morphological productivity of frequent and infrequent words differ, amplifying frequency effects encountered in natural languages \cite{lieberman2007quantifying}.\looseness=-1

The graphotactic condition yields systematically longer codes than the morphological one, although here there are important differences between languages: English, Hebrew and Indonesian have similar code lengths for both code constraints; in the other languages the graphotactic code is substantially longer than the morphological one. 

In all cases, the natural code is longer than the one with both graphotactic and morphological constraints---suggesting languages are not optimally compressed, even when accounting for these constraints.
That said, all of the natural languages are considerably more compressed than a lexicon produced by randomly reassigning wordforms.

\begin{figure}
    \centering
    \includegraphics[width=\columnwidth]{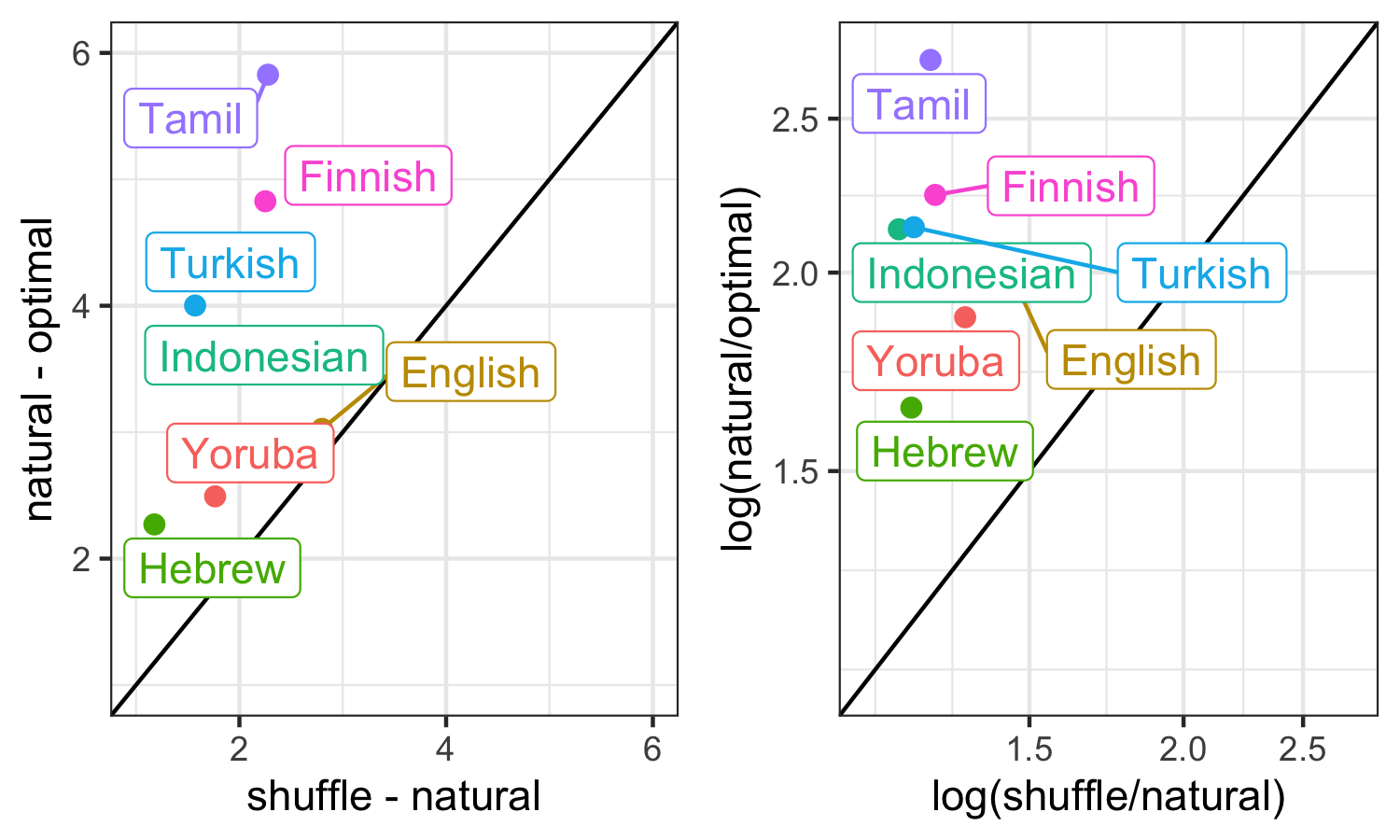}
    \vspace{-20pt}
    \caption{Comparison of the distances---additive (left) and multiplicative (right)---between natural languages and either optimal or shuffled baselines.}
    \label{fig:differences}
    \vspace{-5pt}
\end{figure}

\section{Discussion and Conclusion}

In this paper, we introduced a model-based strategy to assess the relative contribution of different constraints on word (code) length at large. 
In particular, we evaluated how much natural languages differ from systems optimized for Zipf's law of abbreviation. 
Our proposed model improves upon an old method used to consider the efficiency of the lexicon: \defn{random typing models} \citep{miller1957some,moscoso2013missing,ferrer2020optimal}.
\citeauthor{miller1957some} introduced the idea of monkeys typing randomly on a keyboard and analyzed the properties of its resulting language.
The monkeys' text, however, has no morphological or graphotactic constraints \citep[but see][]{caplan2020miller} and does not follow a language's unigram distribution \citep{howes1968zipf}.
As such, it cannot directly encode the same meanings or messages as the original language.

Our results show that, while natural languages do tend to map frequent messages to shorter words, the magnitude of this effect varies widely across our set of  diverse languages. Notably, the distance between natural languages and the optimal codes is larger than the distance between natural languages and their corresponding shuffled code (see \cref{fig:differences}). In other words, natural codes are closer to \emph{not} being optimized (in the Zipfian sense) than to being maximally compressed.

That said, our morphological and graphotactic baselines, when combined, yield codes that display mean code lengths that are (in most cases) closer to the natural code than to the optimal (see \cref{fig:relative}). If our models are indeed able to capture the true patterns in our data, then this means that (compositional) morphology and graphotactics, along with the law of abbreviation, are sufficient to account for most of the length of natural codes---as observed in real languages. 
Graphotactic (primarily) and morphological constraints are enough to derive a code with a similar complexity to that of natural languages, which suggests the other factors discussed above (associated with, e.g., surprisal and non-arbitrary form-meaning mappings) likely play a more modest role in pushing natural languages away from the optimal Zipfian code.

\begin{figure}
    \centering
    \includegraphics[width=\columnwidth]{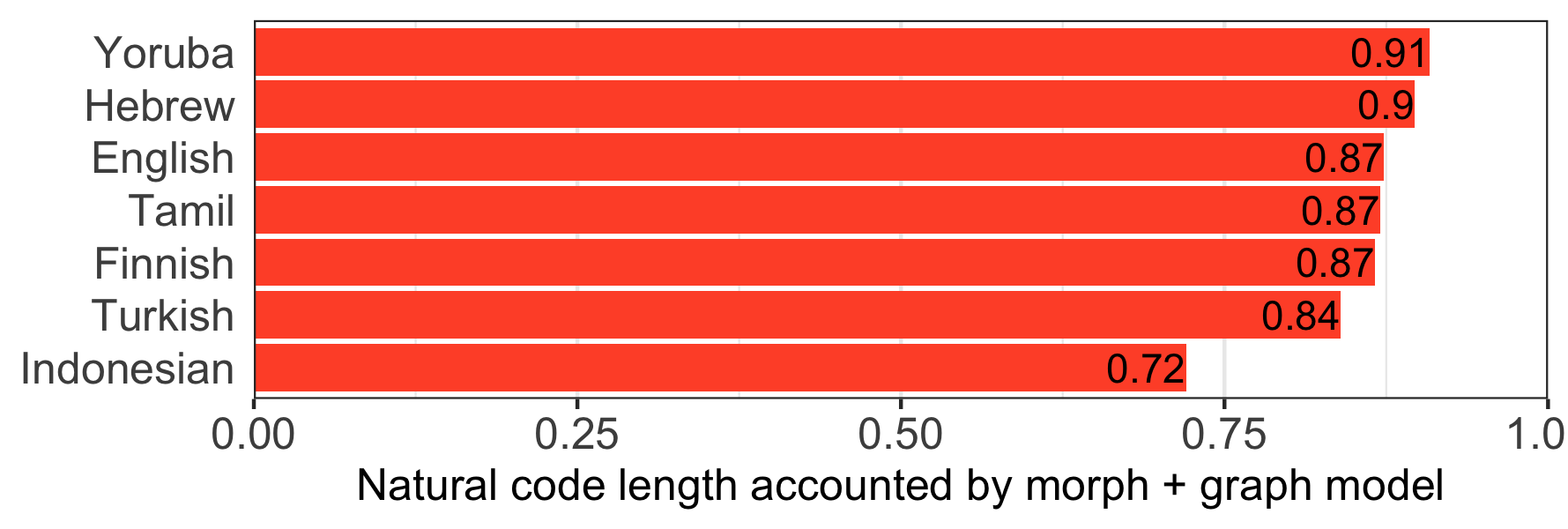}
    \vspace{-17pt}
    \caption{Fraction of code length accounted for by the combined morphology and graphotactics model}
    \label{fig:relative}
    \vspace{-5pt}
\end{figure}

The optimality of the lexicon occupies a major place in the scientific study of the structure and functional evolution of languages \citep{bentz2016zipf,gibson2019efficiency,mahowald2020efficient}. We hope that the method presented here---which allows for a more precise quantification of the (non-)optimality of lexicons---will be used to further the goal of understanding why languages are structured in the ways that they are, while offering insight into the functional tradeoffs that underlie language variation and change.

\section*{Ethical Concerns}

This paper concerns itself with investigating lexicons' optimality under the perspective of Zipf's Law  of Abbreviation.
As we focus on computational linguistic experiments, we see no clear ethical concerns here.
Nonetheless, we note that Wikipedia (from where we collect data) is not a fully representative source of a language's data---the biases in the data will likely also be present in our results.

\section*{Acknowledgements}

Dami\'{a}n E. Blasi acknowledges funding from the Branco Weiss Fellowship, administered by the ETH Z\"{u}rich.
Dami\'{a}n E. Blasi's research was also executed within the framework of the HSE University Basic Research Program and funded by the Russian Academic Excellence Project `5-100'.

\bibliography{acl2020}

\begin{thebibliography}{65}
\expandafter\ifx\csname natexlab\endcsname\relax\def\natexlab#1{#1}\fi

\bibitem[{Antos and Kontoyiannis(2001)}]{antos2001convergence}
András Antos and Ioannis Kontoyiannis. 2001.
\newblock \href {https://doi.org/https://doi.org/10.1002/rsa.10019}
  {Convergence properties of functional estimates for discrete distributions}.
\newblock \emph{Random Structures \& Algorithms}, 19(3‐4):163--193.

\bibitem[{Archer et~al.(2014)Archer, Park, and Pillow}]{archer2014bayesian}
Evan Archer, Il~Memming Park, and Jonathan~W. Pillow. 2014.
\newblock \href {http://jmlr.org/papers/v15/archer14a.html} {{B}ayesian entropy
  estimation for countable discrete distributions}.
\newblock \emph{Journal of Machine Learning Research}, 15(81):2833--2868.

\bibitem[{Aylett and Turk(2004)}]{aylett2004smooth}
Matthew Aylett and Alice Turk. 2004.
\newblock \href {https://journals.sagepub.com/doi/10.1177/00238309040470010201}
  {The smooth signal redundancy hypothesis: {A} functional explanation for
  relationships between redundancy, prosodic prominence, and duration in
  spontaneous speech}.
\newblock \emph{Language and Speech}, 47(1):31--56.

\bibitem[{Bell et~al.(2003)Bell, Jurafsky, Fosler-Lussier, Girand, Gregory, and
  Gildea}]{bell2003effects}
Alan Bell, Daniel Jurafsky, Eric Fosler-Lussier, Cynthia Girand, Michelle
  Gregory, and Daniel Gildea. 2003.
\newblock \href {https://asa.scitation.org/doi/10.1121/1.1534836} {Effects of
  disfluencies, predictability, and utterance position on word form variation
  in {E}nglish conversation}.
\newblock \emph{The Journal of the Acoustical Society of America},
  113(2):1001--1024.

\bibitem[{Bentz and Ferrer{-}i{-}Cancho(2016)}]{bentz2016zipf}
Chris Bentz and Ramon Ferrer{-}i{-}Cancho. 2016.
\newblock \href {http://dx.doi.org/10.15496/publikation-10057} {{Z}ipf's law of
  abbreviation as a language universal}.
\newblock In \emph{Proceedings of the Leiden Workshop on Capturing Phylogenetic
  Algorithms for Linguistics}, pages 1--4, Leiden, Netherlands. University of
  T{\"u}bingen.

\bibitem[{Blasi et~al.(2016)Blasi, Wichmann, Hammarstr{\"o}m, Stadler, and
  Christiansen}]{blasi2016sound}
Dami{\'a}n~E. Blasi, S{\o}ren Wichmann, Harald Hammarstr{\"o}m, Peter~F.
  Stadler, and Morten~H. Christiansen. 2016.
\newblock \href {https://doi.org/10.1073/pnas.1605782113} {Sound--meaning
  association biases evidenced across thousands of languages}.
\newblock \emph{Proceedings of the National Academy of Sciences},
  113(39):10818--10823.

\bibitem[{Blunsom et~al.(2009)Blunsom, Cohn, Goldwater, and
  Johnson}]{blunsom-etal-2009-note}
Phil Blunsom, Trevor Cohn, Sharon Goldwater, and Mark Johnson. 2009.
\newblock \href {https://www.aclweb.org/anthology/P09-2085} {A note on the
  implementation of hierarchical {D}irichlet processes}.
\newblock In \emph{Proceedings of the {ACL}-{IJCNLP} 2009 Conference Short
  Papers}, pages 337--340, Suntec, Singapore. Association for Computational
  Linguistics.

\bibitem[{Boersma(1998)}]{boersma1998functional}
Paul Boersma. 1998.
\newblock \emph{Functional Phonology}.
\newblock Holland Academic Graphics.

\bibitem[{Bostrom and Durrett(2020)}]{bostrom-durrett-2020-byte}
Kaj Bostrom and Greg Durrett. 2020.
\newblock \href {https://doi.org/10.18653/v1/2020.findings-emnlp.414} {Byte
  pair encoding is suboptimal for language model pretraining}.
\newblock In \emph{Findings of the Association for Computational Linguistics:
  EMNLP 2020}, pages 4617--4624, Online. Association for Computational
  Linguistics.

\bibitem[{Caplan et~al.(2020)Caplan, Kodner, and Yang}]{caplan2020miller}
Spencer Caplan, Jordan Kodner, and Charles Yang. 2020.
\newblock \href
  {https://doi.org/https://doi.org/10.1016/j.cognition.2020.104466} {{M}iller's
  monkey updated: {C}ommunicative efficiency and the statistics of words in
  natural language}.
\newblock \emph{Cognition}, 205:104466.

\bibitem[{Casas et~al.(2019)Casas, Hern{\'a}ndez{-}Fern{\'a}ndez, Catal{\`a},
  Ferrer{-}i{-}Cancho, and Baixeries}]{casas2019polysemy}
Bernardino Casas, Antoni Hern{\'a}ndez{-}Fern{\'a}ndez, Neus Catal{\`a}, Ramon
  Ferrer{-}i{-}Cancho, and Jaume Baixeries. 2019.
\newblock \href {https://doi.org/https://doi.org/10.1016/j.csl.2019.03.007}
  {Polysemy and brevity versus frequency in language}.
\newblock \emph{Computer Speech \& Language}, 58:19--50.

\bibitem[{Creutz and Lagus(2007)}]{creutz2007unsupervised}
Mathias Creutz and Krista Lagus. 2007.
\newblock Unsupervised models for morpheme segmentation and morphology
  learning.
\newblock \emph{ACM Transactions on Speech and Language Processing}, 4(1).

\bibitem[{Dautriche et~al.(2018)Dautriche, Fibla, Fievet, and
  Christophe}]{dautriche2018learning}
Isabelle Dautriche, Laia Fibla, Anne-Caroline Fievet, and Anne Christophe.
  2018.
\newblock \href
  {https://www.sciencedirect.com/science/article/pii/S0010028518300136}
  {Learning homophones in context: {E}asy cases are favored in the lexicon of
  natural languages}.
\newblock \emph{Cognitive Psychology}, 104:83--105.

\bibitem[{Dautriche et~al.(2017)Dautriche, Mahowald, Gibson, and
  Piantadosi}]{dautriche2017wordform}
Isabelle Dautriche, Kyle Mahowald, Edward Gibson, and Steven~T. Piantadosi.
  2017.
\newblock \href {https://onlinelibrary.wiley.com/doi/full/10.1111/cogs.12453}
  {Wordform similarity increases with semantic similarity: {A}n analysis of 100
  languages}.
\newblock \emph{Cognitive Science}, 41(8):2149--2169.

\bibitem[{Dingemanse et~al.(2015)Dingemanse, Blasi, Lupyan, Christiansen, and
  Monaghan}]{dingemanse2015arbitrariness}
Mark Dingemanse, Dami{\'a}n~E. Blasi, Gary Lupyan, Morten~H. Christiansen, and
  Padraic Monaghan. 2015.
\newblock \href {https://doi.org/10.1016/j.tics.2015.07.013} {Arbitrariness,
  iconicity, and systematicity in language}.
\newblock \emph{Trends in Cognitive Sciences}, 19(10):603--615.

\bibitem[{Fenk and Fenk(1980)}]{fenk1980konstanz}
August Fenk and Gertraud Fenk. 1980.
\newblock Konstanz im {K}urzzeitgedächtnis - {K}onstanz im sprachlichen
  {I}nformationsfluß?
\newblock \emph{Zeitschrift für Experimentelle und Angewandte Psychologie},
  27(3):400--414.

\bibitem[{Fenk-Oczlon and Fenk(2008)}]{fenk2008complexity}
Gertraud Fenk-Oczlon and August Fenk. 2008.
\newblock Complexity trade-offs between the subsystems of language.
\newblock \emph{Language Complexity: Typology, Contact, Change}, 43:65.

\bibitem[{Ferrer{-}i{-}Cancho et~al.(2020)Ferrer{-}i{-}Cancho, Bentz, and
  Seguin}]{ferrer2020optimal}
Ramon Ferrer{-}i{-}Cancho, Christian Bentz, and Caio Seguin. 2020.
\newblock \href {https://doi.org/10.1080/09296174.2020.1778387} {Optimal coding
  and the origins of {Z}ipfian laws}.
\newblock \emph{Journal of Quantitative Linguistics}, 0(0):1--30.

\bibitem[{Ferrer{-}i{-}Cancho and Vitevitch(2018)}]{ferrer2018origins}
Ramon Ferrer{-}i{-}Cancho and Michael~S. Vitevitch. 2018.
\newblock \href
  {https://asistdl.onlinelibrary.wiley.com/doi/abs/10.1002/asi.24057} {The
  origins of {Z}ipf's meaning-frequency law}.
\newblock \emph{Journal of the Association for Information Science and
  Technology}, 69(11):1369--1379.

\bibitem[{Gage(1994)}]{gage1994new}
Philip Gage. 1994.
\newblock A new algorithm for data compression.
\newblock \emph{The C Users Journal}, 12(2):23–38.

\bibitem[{Gahl(2008)}]{gahl2008time}
Susanne Gahl. 2008.
\newblock \href {https://doi.org/10.1353/lan.0.0035} {Time and thyme are not
  homophones: {T}he effect of lemma frequency on word durations in spontaneous
  speech}.
\newblock \emph{Language}, 84(3):474--496.

\bibitem[{Gibson et~al.(2019)Gibson, Futrell, Piantadosi, Dautriche, Mahowald,
  Bergen, and Levy}]{gibson2019efficiency}
Edward Gibson, Richard Futrell, Steven~P. Piantadosi, Isabelle Dautriche, Kyle
  Mahowald, Leon Bergen, and Roger Levy. 2019.
\newblock \href
  {https://www.cell.com/trends/cognitive-sciences/fulltext/S1364-6613(19)30058-0}
  {How efficiency shapes human language}.
\newblock \emph{Trends in Cognitive Sciences}, 23(5):389--407.

\bibitem[{Goldwater et~al.(2011)Goldwater, Griffiths, and
  Johnson}]{Goldwater2011}
Sharon Goldwater, Thomas~L. Griffiths, and Mark Johnson. 2011.
\newblock \href {https://www.jmlr.org/papers/v12/goldwater11a.html} {Producing
  power-law distributions and damping word frequencies with two-stage language
  models}.
\newblock \emph{Journal of Machine Learning Research}, 12:2335--2382.

\bibitem[{Hochreiter and Schmidhuber(1997)}]{hochreiter1997long}
Sepp Hochreiter and J{\"u}rgen Schmidhuber. 1997.
\newblock \href {https://dl.acm.org/doi/10.1162/neco.1997.9.8.1735} {Long
  short-term memory}.
\newblock \emph{Neural Computation}, 9(8):1735--1780.

\bibitem[{Honnibal et~al.(2020)Honnibal, Montani, Van~Landeghem, and
  Boyd}]{spacy2020}
Matthew Honnibal, Ines Montani, Sofie Van~Landeghem, and Adriane Boyd. 2020.
\newblock \href {https://doi.org/10.5281/zenodo.1212303} {{spaCy}:
  {I}ndustrial-strength natural language processing in python}.

\bibitem[{Howes(1968)}]{howes1968zipf}
Davis Howes. 1968.
\newblock \href {https://www.jstor.org/stable/1421275} {{Z}ipf's law and
  {M}iller's random-monkey model}.
\newblock \emph{The American Journal of Psychology}, 81(2):269--272.

\bibitem[{Huffman(1952)}]{huffman1952method}
David~A. Huffman. 1952.
\newblock \href {https://doi.org/10.1109/JRPROC.1952.273898} {A method for the
  construction of minimum-redundancy codes}.
\newblock \emph{Proceedings of the IRE}, 40(9):1098--1101.

\bibitem[{Jurafsky et~al.(2001)Jurafsky, Bell, Gregory, and
  Raymond}]{jurafsky_probabilistic_2001}
Dan Jurafsky, Alan Bell, Michelle Gregory, and William Raymond. 2001.
\newblock \href {https://psycnet.apa.org/doi/10.1075/tsl.45.13jur}
  {Probabilistic relations between words: {E}vidence from reduction in lexical
  production.}
\newblock In \emph{Frequency and the Emergence of Linguistic Structure}. John
  Benjamins, Amsterdam.

\bibitem[{King and Wedel(2020)}]{king2020greater}
Adam King and Andrew Wedel. 2020.
\newblock \href
  {https://direct.mit.edu/opmi/article/doi/10.1162/opmi_a_00030/95937/Greater-Early-Disambiguating-Information-for-Less}
  {Greater early disambiguating information for less-probable words: {T}he
  lexicon is shaped by incremental processing}.
\newblock \emph{Open Mind}, pages 1--12.

\bibitem[{Levy and Jaeger(2007)}]{levy2007speakers}
Roger~P. Levy and Tim~Florian Jaeger. 2007.
\newblock \href
  {https://papers.nips.cc/paper/2006/hash/c6a01432c8138d46ba39957a8250e027-Abstract.html}
  {Speakers optimize information density through syntactic reduction}.
\newblock In \emph{Advances in Neural Information Processing Systems}, pages
  849--856.

\bibitem[{Lewis and Frank(2016)}]{lewis2016length}
Molly~L. Lewis and Michael~C. Frank. 2016.
\newblock \href
  {https://www.sciencedirect.com/science/article/pii/S0010027716300919} {The
  length of words reflects their conceptual complexity}.
\newblock \emph{Cognition}, 153:182--195.

\bibitem[{Lieberman et~al.(2007)Lieberman, Michel, Jackson, Tang, and
  Nowak}]{lieberman2007quantifying}
Erez Lieberman, Jean{-}Baptiste Michel, Joe Jackson, Tina Tang, and Martin~A.
  Nowak. 2007.
\newblock \href {https://www.nature.com/articles/nature06137} {Quantifying the
  evolutionary dynamics of language}.
\newblock \emph{Nature}, 449(7163):713--716.

\bibitem[{Linder et~al.(1997)Linder, Tarokh, and Zeger}]{linder1997existence}
Tamas Linder, Vahid Tarokh, and Kenneth Zeger. 1997.
\newblock \href {https://ieeexplore.ieee.org/document/641571} {Existence of
  optimal prefix codes for infinite source alphabets}.
\newblock \emph{IEEE Transactions on Information Theory}, 43(6):2026--2028.

\bibitem[{Loshchilov and Hutter(2019)}]{loshchilov2019decoupled}
Ilya Loshchilov and Frank Hutter. 2019.
\newblock \href {https://openreview.net/forum?id=Bkg6RiCqY7} {Decoupled weight
  decay regularization}.
\newblock In \emph{International Conference on Learning Representations}.

\bibitem[{Mahowald et~al.(2020)Mahowald, Dautriche, Braginsky, and
  Gibson}]{mahowald2020efficient}
Kyle Mahowald, Isabelle Dautriche, Mika Braginsky, and Edward Gibson. 2020.
\newblock \href {https://psyarxiv.com/4an6v} {Efficient communication and the
  organization of the lexicon}.
\newblock \emph{PsyArXiv preprint PsyArXiv:4an6v}.

\bibitem[{Mandelbrot(1953)}]{mandelbrot1953informational}
Beno\^{i}t Mandelbrot. 1953.
\newblock An informational theory of the statistical structure of language.
\newblock \emph{Communication Theory}, 84:486--502.

\bibitem[{Miller(1955)}]{miller1955note}
George Miller. 1955.
\newblock Note on the bias of information estimates.
\newblock \emph{Information Theory in Psychology: Problems and Methods}.

\bibitem[{Miller(1957)}]{miller1957some}
George~A. Miller. 1957.
\newblock \href
  {https://www.jstor.org/stable/1419346?origin=crossref&seq=1#metadata_info_tab_contents}
  {Some effects of intermittent silence}.
\newblock \emph{The American Journal of Psychology}, 70(2):311--314.

\bibitem[{Monaghan et~al.(2014)Monaghan, Shillcock, Christiansen, and
  Kirby}]{monaghan_how_2014}
Padraic Monaghan, Richard~C. Shillcock, Morten~H. Christiansen, and Simon
  Kirby. 2014.
\newblock \href {https://doi.org/10.1098/rstb.2013.0299} {How arbitrary is
  language?}
\newblock \emph{Philosophical Transactions of the Royal Society B: Biological
  Sciences}, 369(1651):20130299.

\bibitem[{Moscoso~del Prado(2013)}]{moscoso2013missing}
Fermin Moscoso~del Prado. 2013.
\newblock \href {https://escholarship.org/uc/item/7738n7cz} {The missing
  baselines in arguments for the optimal efficiency of languages}.
\newblock In \emph{Proceedings of the Annual Meeting of the Cognitive Science
  Society}, volume~35.

\bibitem[{Paninski(2003)}]{paninski2003estimation}
Liam Paninski. 2003.
\newblock \href {https://doi.org/10.1162/089976603321780272} {Estimation of
  entropy and mutual information}.
\newblock \emph{Neural Computation}, 15(6):1191--1253.

\bibitem[{Piantadosi et~al.(2011)Piantadosi, Tily, and
  Gibson}]{piantadosi2011word}
Steven~T. Piantadosi, Harry Tily, and Edward Gibson. 2011.
\newblock \href {https://www.pnas.org/content/108/9/3526} {Word lengths are
  optimized for efficient communication}.
\newblock \emph{Proceedings of the National Academy of Sciences},
  108(9):3526--3529.

\bibitem[{Piantadosi et~al.(2012)Piantadosi, Tily, and
  Gibson}]{piantadosi2012communicative}
Steven~T. Piantadosi, Harry Tily, and Edward Gibson. 2012.
\newblock \href
  {https://www.sciencedirect.com/science/article/pii/S0010027711002496} {The
  communicative function of ambiguity in language}.
\newblock \emph{Cognition}, 122(3):280--291.

\bibitem[{Pilehvar and Camacho{-}Collados(2020)}]{pilehvar2020embeddings}
Mohammad~Taher Pilehvar and Jose Camacho{-}Collados. 2020.
\newblock Embeddings in natural language processing: {T}heory and advances in
  vector representations of meaning.
\newblock \emph{Synthesis Lectures on Human Language Technologies},
  13(4):1--175.

\bibitem[{Pimentel et~al.(2021{\natexlab{a}})Pimentel, Cotterell, and
  Roark}]{pimentel-etal-2021-disambiguatory}
Tiago Pimentel, Ryan Cotterell, and Brian Roark. 2021{\natexlab{a}}.
\newblock \href {https://arxiv.org/abs/2102.02183} {Disambiguatory signals are
  stronger in word-initial positions}.
\newblock In \emph{Proceedings of the 16th Conference of the {E}uropean Chapter
  of the Association for Computational Linguistics: Volume 1, Long Papers}.
  Association for Computational Linguistics.

\bibitem[{Pimentel et~al.(2020{\natexlab{a}})Pimentel, Hall~Maudslay, Blasi,
  and Cotterell}]{pimentel-etal-2020-speakers}
Tiago Pimentel, Rowan Hall~Maudslay, Dami\'{a}n Blasi, and Ryan Cotterell.
  2020{\natexlab{a}}.
\newblock \href {https://doi.org/10.18653/v1/2020.emnlp-main.328} {Speakers
  fill lexical semantic gaps with context}.
\newblock In \emph{Proceedings of the 2020 Conference on Empirical Methods in
  Natural Language Processing (EMNLP)}, pages 4004--4015, Online. Association
  for Computational Linguistics.

\bibitem[{Pimentel et~al.(2019)Pimentel, McCarthy, Blasi, Roark, and
  Cotterell}]{pimentel2019meaning}
Tiago Pimentel, Arya~D. McCarthy, Dami\'{a}n Blasi, Brian Roark, and Ryan
  Cotterell. 2019.
\newblock \href {https://doi.org/10.18653/v1/P19-1171} {Meaning to form:
  {M}easuring systematicity as information}.
\newblock In \emph{Proceedings of the 57th Annual Meeting of the Association
  for Computational Linguistics}, pages 1751--1764, Florence, Italy.
  Association for Computational Linguistics.

\bibitem[{Pimentel et~al.(2020{\natexlab{b}})Pimentel, Roark, and
  Cotterell}]{pimentel-etal-2020-phonotactics}
Tiago Pimentel, Brian Roark, and Ryan Cotterell. 2020{\natexlab{b}}.
\newblock \href {https://doi.org/10.1162/tacl\_a\_00296} {Phonotactic
  complexity and its trade-offs}.
\newblock \emph{Transactions of the Association for Computational Linguistics},
  8:1--18.

\bibitem[{Pimentel et~al.(2021{\natexlab{b}})Pimentel, Roark, Wichmann,
  Cotterell, and Blasi}]{pimentel-etal-2021-finding}
Tiago Pimentel, Brian Roark, S{\o}ren Wichmann, Ryan Cotterell, and Dami\'{a}n
  Blasi. 2021{\natexlab{b}}.
\newblock \href {https://arxiv.org/abs/2104.06325} {Finding concept-specific
  biases in form--meaning associations}.
\newblock In \emph{Proceedings of the 2021 Conference of the North {A}merican
  Chapter of the Association for Computational Linguistics: Human Language
  Technologies, Volume 1 (Long and Short Papers)}, Virtual. Association for
  Computational Linguistics.

\bibitem[{Pitman and Yor(1997)}]{Pitman1997}
Jim Pitman and Marc Yor. 1997.
\newblock \href {https://doi.org/10.1214/aop/1024404422} {The two-parameter
  {P}oisson--{D}irichlet distribution derived from a stable subordinator}.
\newblock \emph{Annals of Probability}, 25(2):855--900.

\bibitem[{Schuster and Nakajima(2012)}]{schuster2012japanese}
Mike Schuster and Kaisuke Nakajima. 2012.
\newblock \href {https://ieeexplore.ieee.org/document/6289079} {{J}apanese and
  {K}orean voice search}.
\newblock In \emph{IEEE International Conference on Acoustics, Speech and
  Signal Processing (ICASSP)}, pages 5149--5152. IEEE.

\bibitem[{Sennrich et~al.(2016)Sennrich, Haddow, and
  Birch}]{sennrich-etal-2016-neural}
Rico Sennrich, Barry Haddow, and Alexandra Birch. 2016.
\newblock \href {https://doi.org/10.18653/v1/P16-1162} {Neural machine
  translation of rare words with subword units}.
\newblock In \emph{Proceedings of the 54th Annual Meeting of the Association
  for Computational Linguistics (Volume 1: Long Papers)}, pages 1715--1725,
  Berlin, Germany. Association for Computational Linguistics.

\bibitem[{Shannon(1948)}]{shannon1948mathematical}
Claude~E. Shannon. 1948.
\newblock A mathematical theory of communication.
\newblock \emph{Bell System Technical Journal}, 27(3):379--423.

\bibitem[{Sinclair(1999)}]{sinclair1999way}
John Sinclair. 1999.
\newblock A way with common words.
\newblock \emph{Language and Computers}, 26:157--180.

\bibitem[{Smit et~al.(2014)Smit, Virpioja, Gr{\"o}nroos, and
  Kurimo}]{smit-etal-2014-morfessor}
Peter Smit, Sami Virpioja, Stig-Arne Gr{\"o}nroos, and Mikko Kurimo. 2014.
\newblock \href {https://doi.org/10.3115/v1/E14-2006} {{M}orfessor 2.0:
  {T}oolkit for statistical morphological segmentation}.
\newblock In \emph{Proceedings of the Demonstrations at the 14th Conference of
  the {E}uropean Chapter of the Association for Computational Linguistics},
  pages 21--24, Gothenburg, Sweden. Association for Computational Linguistics.

\bibitem[{Teh(2006)}]{Teh2006}
Yee~Whye Teh. 2006.
\newblock \href {https://doi.org/10.3115/1220175.1220299} {{A hierarchical
  Bayesian language model based on Pitman-Yor processes}}.
\newblock In \emph{COLING/ACL 2006 - 21st International Conference on
  Computational Linguistics and 44th Annual Meeting of the Association for
  Computational Linguistics, Proceedings of the Conference}, volume~1, pages
  985--992, Morristown, NJ, USA. Association for Computational Linguistics.

\bibitem[{Trott and Bergen(2020)}]{trott2020human}
Sean Trott and Benjamin Bergen. 2020.
\newblock \href
  {https://www.sciencedirect.com/science/article/pii/S0010027720302687} {Why do
  human languages have homophones?}
\newblock \emph{Cognition}, 205:104449.

\bibitem[{van Son and Pols(2003{\natexlab{a}})}]{son2003efficient}
Rob J. J.~H. van Son and Louis C.~W. Pols. 2003{\natexlab{a}}.
\newblock How efficient is speech?
\newblock In \emph{Proceedings of the Institute of Phonetic Sciences},
  volume~25, pages 171--184.

\bibitem[{van Son and Pols(2003{\natexlab{b}})}]{son2003information}
Rob J. J.~H. van Son and Louis C.~W. Pols. 2003{\natexlab{b}}.
\newblock \href
  {https://www.isca-speech.org/archive/archive_papers/eurospeech_2003/e03_0769.pdf}
  {Information structure and efficiency in speech production}.
\newblock In \emph{Eighth European Conference on Speech Communication and
  Technology}, Geneva, Switzerland.

\bibitem[{Vitevitch and Luce(1999)}]{vitevitch1999probabilistic}
Michael~S. Vitevitch and Paul~A. Luce. 1999.
\newblock \href
  {https://www.sciencedirect.com/science/article/pii/S0749596X98926183}
  {Probabilistic phonotactics and neighborhood activation in spoken word
  recognition}.
\newblock \emph{Journal of Memory and Language}, 40(3):374--408.

\bibitem[{Wei and Tanner(1990)}]{Wei1990}
Greg C.~G. Wei and Martin~A. Tanner. 1990.
\newblock \href {https://doi.org/10.1080/01621459.1990.10474930} {A {M}onte
  {C}arlo implementation of the {EM} algorithm and the poor man's data
  augmentation algorithms}.
\newblock \emph{Journal of the American Statistical Association},
  85(411):699--704.

\bibitem[{Welleck et~al.(2020)Welleck, Kulikov, Kim, Pang, and
  Cho}]{welleck-etal-2020-consistency}
Sean Welleck, Ilia Kulikov, Jaedeok Kim, Richard~Yuanzhe Pang, and Kyunghyun
  Cho. 2020.
\newblock \href {https://doi.org/10.18653/v1/2020.emnlp-main.448} {Consistency
  of a recurrent language model with respect to incomplete decoding}.
\newblock In \emph{Proceedings of the 2020 Conference on Empirical Methods in
  Natural Language Processing (EMNLP)}, pages 5553--5568, Online. Association
  for Computational Linguistics.

\bibitem[{Xu et~al.(2020)Xu, Duong, Malt, Jiang, and
  Srinivasan}]{xu2020conceptual}
Yang Xu, Khang Duong, Barbara~C. Malt, Serena Jiang, and Mahesh Srinivasan.
  2020.
\newblock \href
  {https://www.sciencedirect.com/science/article/pii/S0010027720300998}
  {Conceptual relations predict colexification across languages}.
\newblock \emph{Cognition}, 201:104280.

\bibitem[{Zipf(1935)}]{zipf1935psycho}
George~Kingsley Zipf. 1935.
\newblock \emph{The Psycho-biology of Language: {A}n Introduction to Dynamic
  Philology}.
\newblock Houghton Mifflin.

\bibitem[{Zipf(1949)}]{zipf1949human}
George~Kingsley Zipf. 1949.
\newblock \emph{Human Behavior and the Principle of Least Effort}.
\newblock Addison-Wesley Press.

\end{thebibliography}
\bibliographystyle{acl_natbib}

\appendix

\section*{Appendix}
\section{Dataset sizes} \label{sec:dataset_sizes}

In this section, we present the number of word tokens (\cref{tab:dataset_tokens}) and word types (\cref{tab:dataset_types}) in our analyzed datasets.

\begin{table}[ht]
    \centering
\resizebox{\columnwidth}{!}{%
    \begin{tabular}{lrrr}
        \toprule
        & \multicolumn{1}{c}{Train} & \multicolumn{1}{c}{Validation} & \multicolumn{1}{c}{Test} \\
        \midrule
        English & 4,630,371 & 578,510 & 578,796 \\
        Finnish & 2,558,634 & 319,546 & 320,716 \\
        Hebrew & 4,911,953 & 613,457 & 609,864 \\
        Indonesian & 4,039,552 & 506,085 & 507,587 \\
        Tamil & 3,286,075 & 412,776 & 412,416 \\
        Turkish & 2,676,471 & 333,120 & 332,359 \\
        Yoruba & 373,517 & 46,415 & 46,283 \\
        \bottomrule
    \end{tabular}
}
    \caption{The number of word tokens used in training, validation and testing.}
    \label{tab:dataset_tokens}
\end{table}

\begin{table}[ht]
    \centering
\resizebox{\columnwidth}{!}{%
    \begin{tabular}{lrrr}
        \toprule
        & \multicolumn{1}{c}{Train} & \multicolumn{1}{c}{Validation} & \multicolumn{1}{c}{Test} \\
        \midrule
        English & 242,030 & 66,668 & 66,243 \\
        Finnish & 466,745 & 109,232 & 110,378 \\
        Hebrew & 311,860 & 104,555 & 104,478 \\
        Indonesian & 243,118 & 69,792 & 70,079 \\
        Tamil & 479,668 & 116,196 & 115,422 \\
        Turkish & 308,419 & 84,300 & 83,871 \\
        Yoruba & 47,740 & 12,877 & 12,877 \\
        \bottomrule
    \end{tabular}
}
    \caption{The number of word types used in training, validation and testing.}
    \label{tab:dataset_types}
    \vspace{-5pt}
\end{table}

\section{Model Training} \label{sec:model_training}

As mentioned in the main text, we cannot directly infer the parameters of our model and we use a solution similar to expectation maximization \cite{Wei1990}. 
We freeze our LSTM generator while learning the PYCRP parameters and then freeze the PYCRP to train the LSTM model.

\paragraph{Expectation step.} This step uses a Gibbs sampling procedure to estimate the parameters of the PYCRP. 
For each token in our dataset, we fix all cluster assignments $\bz_{-n}$ except for the given token's one $z_n$. We then re-sample this token's cluster based on the marginal probability $p(z_n | \bz_{-n}, \bl, \bw_n)$.
We do this for 5 epochs, and use the assignments which result in the best development set cross-entropy.
This process can both remove clusters and create new ones by replacing tokens. 
The set of populated clusters (together with their wordform labels) then allows creating a new wordform dataset of size $K'$, where the distribution of the token frequencies is expected to be less skewed. 
In practice, this wordform dataset is thus created from the resulting set of cluster labels $\{\bl_k\}_{k=1}^{K'}$, i.e. a word will appear in this new dataset as many times as it was assigned as a cluster label.

\paragraph{Maximization step.} 
We use the set of populated cluster labels 
to train the generator LSTM---assuming 
that this allows learning a more representative model of a language's graphotactics as the irregular common words are less dominant in its training set.
In other words, at each epoch, the generator will be trained in a wordform as many times as it has been assigned as a cluster label.

\paragraph{Hyperparameters and implementation details.}
For the PYCRP, we fix hyper-parameters $a=0.5$ and $b=10{,}000$, and we use the optimized Gibbs sampling algorithm designed by \citet{blunsom-etal-2009-note}.
As our generator, we use a three layers LSTM with an embedding size of 128, a hidden size of 512 and dropout of $.33$. This LSTM is trained using AdamW \citep{loshchilov2019decoupled} 
and we hotstart it by initially training on the set of word types in the training set (the set of unique wordforms in it).

\section{Proof of \cref{prop:bound_entropy}} \label{app:proof_entropy}

We present here the proof of \cref{prop:bound_entropy}. This proposition is repeated here for convenience:

\paragraph{\cref{prop:bound_entropy}.}
\textit{If a language model $p(\bw)$ is $\varepsilon$-smooth, then its entropy is finite, i.e.
$\ent(W) < \infty$.}
\begin{proof}
To prove this, we will break the entropy of a string in two parts. The entropy of the first character, plus the entropy of the following ones given the first, as in:
\begin{equation}
    \entropy(W) = \entropy(W_1) + \entropy(W_{>1} \mid W_1)
\end{equation}
We first bound the entropy of the first character using a uniform distribution upperbound:
\begin{align}
    \entropy(W_1) &= -\sum_{\bw_1 \in \Sigma} p(\bw_1) \log p(\bw_1) \\
    &\le - \log |\Sigma|   \nonumber
\end{align}
We now use the $\varepsilon$-smoothness property to upperbound the entropy of the following characters:
\begin{align}
    \entropy(& W_{>1} \mid W_1) \\
    &= \sum_{\bw_1 \in \Sigma} p(\bw_1) \entropy(W_{>1} \mid W_1 = \bw_1) \nonumber \\
    &= p(\STOP) \entropy(W_{>1} \mid W_1 = \STOP)  \nonumber \\
    & \quad +\sum_{\bw_1 \in \Sigma, \bw != \STOP} p(\bw_1) \entropy(W_{>1} \mid W_1 = \bw_1)  \nonumber \\
    &= \sum_{\bw_1 \in \Sigma, \bw != \STOP} p(\bw_1) \entropy(W_{>1} \mid W_1 = \bw_1)  \nonumber \\
    &\le \sum_{\bw_1 \in \Sigma, \bw != \STOP} p(\bw_1) \entropy(W_{>1})  \nonumber \\
    &\le \left(1 - \varepsilon \right) \entropy(W)  \nonumber
\end{align}
Given both these upperbounds, we can bound the full wordform entropy:
\begin{align}
    \entropy(W) &= \entropy(W_1) + \entropy(W_{>1} \mid W_1) \\
    &\le - \log |\Sigma| 
     + \left(1 - \varepsilon \right) \entropy(W)  \nonumber
\end{align}
Finally, with simple algebraic manipulations we complete the proof:
\begin{align}
    \entropy(W) &\le - \frac{1}{\varepsilon} \log |\Sigma|
\end{align}

\end{proof}

\end{document}